%% file: main.tex
\documentclass[11pt,letterpaper]{article}

\usepackage[margin=1in]{geometry}
\usepackage{microtype}
\usepackage{graphicx}
\usepackage{subfigure}
\usepackage{booktabs} 

\usepackage{hyperref}
\usepackage[authoryear,numbers,semicolon,square]{natbib}

\hypersetup{
	colorlinks   = true, 
	urlcolor     = blue, 
	linkcolor    = blue, 
	citecolor   = blue 
}

\usepackage{times,fancyhdr,xcolor,algorithm,algorithmic,eso-pic,forloop,url}

\usepackage{amsmath}
\usepackage{amssymb}
\usepackage{mathtools}
\usepackage{amsthm}

\usepackage[capitalize,noabbrev]{cleveref}

\theoremstyle{plain}
\newtheorem{thm}{Theorem}[section]

\newtheorem{lem}[thm]{Lemma}

\theoremstyle{definition}

\theoremstyle{remark}

\newcommand{\algonamell}{LLL}

\include{math_commands}

\title{Nearly Minimax Algorithms for \\ Linear Bandits with Shared Representation}

\author{
	Jiaqi Yang \\
	University of California, Berkeley \\
	\texttt{yjq@berkeley.edu} \\
	\\
	Jason D.\ Lee \\
	Princeton University \\
	\texttt{jasonlee@princeton.edu}
	\and
	Qi Lei\\
	Princeton University\\
	\texttt{qilei@princeton.edu} \\
	\\
	Simon S.\ Du \\ 
	University of Washington \\
	\texttt{ssdu@cs.washington.edu}
}
\allowdisplaybreaks

\begin{document}

\maketitle

\begin{abstract}
We give novel algorithms for multi-task and lifelong linear bandits with shared representation. Specifically, we consider the setting where we play $M$ linear bandits with dimension $d$, each for $T$ rounds, and these $M$ bandit tasks share a common $k(\ll d)$ dimensional linear representation. For both the multi-task setting where we play the tasks concurrently, and the lifelong setting where we play tasks sequentially, we come up with novel algorithms that achieve $\widetilde{O}\left(d\sqrt{kMT} + kM\sqrt{T}\right)$ regret bounds, which matches the known minimax regret lower bound up to logarithmic factors and closes the gap in existing results \citep{yang2021impact}. Our main technique include a more efficient estimator for the low-rank linear feature extractor and an accompanied novel analysis for this estimator. 
\end{abstract}
\section{Introduction}
\label{sec:intro}
In this paper, we give nearly minimax optimal algorithms for multi-task  linear bandits with shared representation.
Multi-task representation learning learns a joint low-dimensional feature extractor from different but related tasks, so the composition of this feature extractor and a simple function (e.g., linear function) can give more accurate predictions than the standard single-task learning paradigm~\citep{baxter2000model,caruana1997multitask}.
The fundamental reason for this improvement is that the relatedness among tasks make us learn the joint feature extractor more efficiently than treating each task independently.

Empirically,  representation learning has led to successes in applications such as computer vision~\citep{li2014joint}, natural language processing~\citep{ando2005framework,liu2019towards}, and drug discovery~\citep{ramsundar2015massively}.
Recently, representation learning has become a powerful approach in sequential decision-making problems, such as bandits and reinforcement learning \citep{teh2017distral,taylor2009transfer,lazaric2011transfer,rusu2015policy,liu2016decoding,parisotto2015actor,higgins2017darla,hessel2019multi,arora2020provable,d'eramo2020sharing}.

While representation learning has become a standard paradigm in a variety of applications~\citep{bengio2013representation}, our theoretical understanding is still far from complete.
This paper concerns the theoretical benefits of representation learning in the linear bandits setting.
Linear bandits is one of the most fundamental and widely-studied settings in sequential decision-making problems, and has profound applications in  clinical treatment, manufacturing process, job scheduling, recommendation systems, etc~\citep{dani2008stochastic,chu2011contextual}.

The formal theoretical study on multi-task linear bandits with shared representation was initiated by \citet{yang2021impact} who studied both finite-action and infinite-action settings.
For the finite-action setting, they gave an algorithm with a near-optimal regret bound.\footnote{
	Throughout this paper, we say a bound is near-optimal if it is matches the minimax lower bound up to logarithmic factors.
}
For the infinite-action setting, there is $\sqrt{dk}$ gap between their upper bound and their minimax lower bound, where $d$ is the feature dimension and $k$ is the dimension of the representation.
Their work was later extended to settings with more general actions and the linear Markov Decision Process setting~\citep{hu2021near}.
However, the $\sqrt{dk}$ factor gap remains open.
In this paper, we address the following fundamental question:
\begin{center}
	\textbf{What is the fundamental limit of linear bandits with shared representation?}
\end{center}
We address this problem and close the gap in previous theoretical analyses based on a novel estimator of the representation.

\subsection{Main Contributions}
We summarize our contributions below.

\begin{itemize}
	\item First, we propose a new algorithm for multi-task infinite-action linear bandits with shared representation.
	Our main technical innovation is a new estimator  based on the singular value decomposition (SVD) on rectangular matrix, in contrast to the estimator in \citep{yang2021impact} which is on a squared matrix.
	Fundamentally, our new estimator has a lower variance. Together with a refined matrix perturbation analysis, we obtain  a smaller regret bound.
	Theoretically, we show our algorithm enjoys an $\widetilde{O}\left(d\sqrt{kMT} + kM\sqrt{T}\right)$ regret where $d$ is the ambient feature dimension, $k$ is the dimension of the representation, $M$ is the number of tasks, and $T$ is the number of rounds. This bound improves a $\sqrt{d}$ factor in \citep{yang2021impact} and, importantly, matches the minimax lower bound $\Omega\left(d\sqrt{kMT} + kM\sqrt{T}\right)$ in \citep{yang2021impact} up to logarithmic factors.
	\item Second, we study a novel setting, lifelong linear bandits with shared representation where we need to solve each task sequentially in contrast to the multi-task setting where we solve all tasks concurrently. Our proposed new setting is a natural extension of the recent lifelong representation learning studied in \citep{cao2021provable}.
	We adopt the similar ideas in our approach for multi-task linear bandits we propose a algorithm, \algonamell.
	We prove this algorithm also enjoys an $\widetilde{O}\left(d\sqrt{kMT} + kM\sqrt{T}\right)$ regret bound, which is again nearly minimax optimal. We also consider the pure exploration in lifelonear linear bandits, and obtain a similar near-optimal bound.

	\item Lastly, to evaluate of our proposed new algorithms, we conduct experiments on synthetic data  to compare with prior algorithms. Our experiments show our new algorithms  outperform existing ones, sometime by a large margin.
\end{itemize}

\subsection{Organization}
This paper is organized as follows. In Section~\ref{sec:rel}, we discuss related work. 
In Section~\ref{sec:pre}, we introduce necessary notations, state our main assumptions and formally describe the settings for multi-task linear bandits and life long linear bandits.
In Section~\ref{sec:multi_task}, we describe our near-optimal algorithm for the multi-task linear bandits, and its theoretical analysis.
In Section~\ref{sec:life_long}, we describe our near-optimal algorithm for the lifelong linear bandits, and its theoretical analysis.
In Section~\ref{sec:exp}, we provide empirical evaluations to demonstrate the effectiveness of proposed methods.
We conclude in Section~\ref{sec:con}, and defer some technical lemmas to appendix.

\section{Related Work}
\label{sec:rel}

We focus on the related theoretical results. 
\paragraph{Representation learning.} Multi-task representation learning has been extensively studied in the supervised learning setting with different assumptions \citep{baxter2000model,ando2005framework,ben2003exploiting,maurer2006bounds,cavallanti2010linear,maurer2016benefit,du2020few,tripuraneni2020provable,thekumparampil2021sample,tripuraneni2020theory}
A common and necessary assumption is the existence of a shared representation among all tasks. This is also adopted in work on bandits and reinforcement learning, including the current work.
Besides this assumptions, in order to make the learned representation useful for new tasks,  some other assumptions are needed, including. e.g., the i.i.d. tasks assumption~\citep{maurer2016benefit} and the diversity assumption~\citep{du2020few,tripuraneni2020theory}.
It is worth mentioning that the estimator used in \citep{yang2021impact} for the infinite-arm multi-task linear bandits is based on the method in learning linear representation paper in the supervised learning \citep{tripuraneni2020provable}.
In this paper, we develop a new estimator to improve the regret bound in \citep{yang2021impact}.
We also note that there are analyses for other representation learning schemes beyond multi-task representation learning~\citep{arora2019theoretical,mcnamara2017risk,galanti2016theoretical,alquier2016regret,denevi2018incremental}.

The benefit of representation learning has been studied in sequential decision-making problems.
\citet{d'eramo2020sharing,arora2020provable} used a probabilistic assumption similar to that in \citep{maurer2016benefit} to show representation learning is beneficial for imitation learning.
Meta-learning is closely related to representation learning, although the assumption is not the same \citep{denevi2019learning,finn2019online,khodak2019adaptive,lee2019meta,bertinetto2018meta}.
There are also other works on different directions such linear MDP with a generative model~\citep{lu2021power}, non-stationary sequential learning~\citep{qin2022non}, and linear dynamical systems~\citep{modi2021joint}.

\paragraph{Bandits.}
Linear bandits have been widely studied in recent years \citep{auer2002using, dani2008stochastic, rusmevichientong2010linearly, abbasi2011improved, chu2011contextual, li2019nearly,li2019tight}.
In this paper, we focus on the infinite-action setting in which the $\widetilde{\Theta}(d \sqrt{N})$ regret bound has been shown to be near-optimal~\citep{dani2008stochastic, rusmevichientong2010linearly, li2019tight}.

Our setting is also related to the recent line of work on low-rank bandits~\citep{lale2019stochastic,lu2020low,jun2019bilinear,lattimore2021bandit,huang2021optimal}, because our formulation also admits a low-rank structure. 
However, the approach we use and their are very different.

There are other papers also studied multi-task bandits~\citep{deshmukh2017multi,bastani2019meta,soaremulti}, but their settings are different from ours.

\paragraph{Most related works.}
The most related work is \citep{yang2021impact} which studied the same setting, infinite-action multi-task linear bandits. 
They presented three-stage algorithm which enjoys an $\widetilde{O}\left(d^{1.5}k\sqrt{MT}+kM\sqrt{T}\right)$ regret and derived a lower bound $\Omega\left(d\sqrt{kMT}+kM\sqrt{T}\right)$.
Our algorithm's structure is similar to theirs but we use a more efficient estimator for the representation. Our upper bound improves theirs and matches their lower bound up to logarithmic factors.

\citet{hu2021near} also studied infinite-action multi-task linear bandits, and proposed an algorithm with $\widetilde{O}\left(d\sqrt{kMT}+M\sqrt{dkT}\right)$ where the second term does not match the lower bound. Furthermore, their algorithm is not computationally efficient.
However, we note that our algorithm does not cover their setting because the assumption on the action set is different.
They also extended their algorithm to reinforcement learning with linear function approximation.

\citet{cao2021provable} proposed the problem of lifelong learning of representations. The formulated the problem in a supervised learning manner. Our lifelong linear bandits setting is inspired by their work and our algorithm is  different.

\section{Preliminaries}
\label{sec:pre}
\subsection{Notations.}

For $n \in \sN$, we denote $[n] = \{1, \cdots, n\}$. We denote vectors by bold lowercase characters and matrices by bold uppercase ones. For a vector $\va$, we denote its $i$-th entry by $\va(i)$. For a matrix $\mA$, we denote its $i$-th column by $\mA(i)$, we denote its transpose by $\mA^\top$. We denote its projection matrix by $\mP_\mA = \mA(\mA^\top \mA)^{-1} \mA^\top$, which projects a vector to the column span of $\mA$. We denote $\mP_\mA^\perp = \mI - \mP_\mA$, which projects a vector to the orthogonal complement of the column span of $\mA$. For a (possibly infinite) sequence of random variables $X_1, \cdots, X_n, \cdots$, we denote by $\sigma(X_1, \cdots, X_n, \cdots)$ the $\sigma$-field it generates. We use $\mathbb{B}_r(d)$ to denote the radius-$r$ ball in a $d$-dimensional space, and $\sS_r(d-1)$ to denote the radius-$r$ sphere. We use $\Unif$ to denote the uniform distribution and $\gN$ to denote Gaussian distribution. $\tildeO(\cdot)$ omits logarithmic factors.

\subsection{Settings.}

In this paper, we study two settings for linear bandits with shared representation. These two settings have some similarities. Below, we first state the shared components of the two settings, then we discuss their differences respectively. 

We play $M$ linear bandits tasks of dimension $d$, each for $T$ time steps. For each $m \in [M]$, task $m$ has a linear coefficient $\vtheta_m \in \sR^d$ with $\norm{\vtheta_m} = 1$, and an action set $\gA_m =\{\va \in \sR^d : \norm{\va} \le 1\}$. We use $\mTheta = [\vtheta_1, \cdots, \vtheta_M]\in \R^{d\times M}$ to denote the matrix of all linear coefficients. 

We define $k = \rank \mTheta$. We assume that $\mTheta = \mB \mW$ for some matrices $\mB \in \sR^{d \times k}, \mW \in \sR^{k \times M}$. Without loss of generality, we assume that $\mB$ has orthonormal columns. Practically, $\mB$ can be seen as a common linear feature extractor shared across all $M$ tasks.

We emphasize that the existence of $\mB$ is, to some degree, necessary to ensure the benefits of representation learning.

\paragraph{Multi-Task Liner Bandits with Shared Representation.} For the multi-task setting, the player solves all bandits tasks concurrently. 
The interactive protocol is as follows. At each time step $t = 1, \cdots, T$, for each task $m \in [M]$, the selects an action $\va_{t, m} \in \gA_m$. after the player commits the batch of actions $\{\va_{t, m} : m \in [M]\}$, it receives the batch of rewards $\{r_{t, m} : m \in [M]\}$, where 
\begin{align}
	r_{t, m} = \langle \va_{t, m}, \vtheta_{m} \rangle + \eta_{t, m}, \qquad \eta_{t, m} \sim_{\mathrm{iid.}} \gN(0, 1), \label{eq:reward}
\end{align}
and the goal is to minimize the total expected regret 
\begin{align}
	\E \Reg^{T, M} = \sum_{t = 1}^T \sum_{m = 1}^M \max_{\va \in \gA_m} \langle \va - \va_{t, m}, \vtheta_{m} \rangle. \label{eq:regret}
\end{align}

\paragraph{Lifelong Linear Bandits with Shared Representation.} For the lifelong setting, the player solves all bandits tasks sequentially. Specifically, the interactive protocol is as follows. The tasks arrive sequentially in the order $m = 1, \cdots, M$. When task $m$ arrives, the player is required to interact with it for $t$ steps. At each time step $t = 1, \cdots, T$, the player selects an action $\va_{t, m} \in \gA_m$ and receives the reward defined in (\ref{eq:reward}). 
There are two goals for lifelong linear bandits which we will tackle in this paper:
\begin{itemize}
	\item \textbf{Regret minimization}: minimize the total expected regret defined in (\ref{eq:regret}).
	\item \textbf{Pure exploration}: use as few as possible interactions to make $\norm{\hatvtheta_m - \vtheta_m} \le \epsilon$ for some given accuracy parameter $\epsilon > 0$ and all $m \in [M]$, with high probability $1-\delta$.
\end{itemize}

\section{Main Results for Multi-Task Learning Bandits with Shared Representation}
\label{sec:multi_task}
\begin{algorithm}[t] 
	\caption{Nearly Minimax Multi-Task Linear Bandits with Shared Representation}\label{alg:etc}
	\begin{algorithmic}[1]
		\STATE {\bf \underline{Input:}}   Ambient dimension $d$, representation dimension $k$, number of tasks $M$, time horizon $T$.
		\STATE {\bf \underline{Initialize:}} Set $T_1 = \lceil d \sqrt{\frac{kT}{M}} \rceil, T_2 = \lceil k \sqrt T \rceil$.
		\STATE {\bf \underline{Stage 1: Subspace Exploration}} 
		\FOR{time step $t \gets 1$ to $T_1$} 
		\STATE Generate $\va_{t, m} \sim \Unif\left(\mathbb{S}_1(d-1)\right)$ for each task $m \in [M]$. 
		\STATE Commit the batch of actions $\{\va_{t, m} : m \in [M]\}$, receive the batch of rewards $\{r_{t, m} : m \in [M]\}$.  
		\ENDFOR
		\STATE Compute $\hatvtheta_m \gets \frac{d}{T_1} \sum_{t=1}^{T_1} r_{t, m} \va_{t, m}$.
		\STATE Let $\hatmTheta =[\hatvtheta_1, \cdots, \hatvtheta_m] \in \sR^{d \times M}$.
		\STATE Compute the singular value decomposition of $ \hatmTheta$.
		\STATE Let $\hatmB$ be the top-$k$ left singular vectors of $\hatmTheta$.
		\STATE { \bf \underline{Stage 2: Per-Task Exploration}} 
		\FOR {$i \gets 1$ to $k$}
		\STATE Let $\va_{t, m} \gets \hatmB(i)$ for $t = T_1 + b(i-1) + 1, \cdots, T_1 + bi$, where $b = \frac{T_2}{k}$.
		\ENDFOR 
		\FOR {time step $t \gets T_1 + 1, \cdots, T_1 + T_2$}
		\STATE Commit the batch of actions $\{\va_{t, m} : m \in [M]\}$, receive the batch of rewards $\{r_{t, m} : m \in [M]\}$.  
		\ENDFOR
		\STATE Let $\hatmW \in \sR^{k \times M}$. 
		\FOR {$m \gets 1$ to $M$}
		\STATE Compute $\hatmW(m) \gets \argmin_{\vw \in \sR^k} \frac{1}{T_2} \sum_{t = T_1+1}^{T_1+T_2} (\langle \va_{t, m}, \hatmB \vw \rangle - r_{t, m})^2$.
		\ENDFOR
		\STATE { \bf \underline{Stage 3: Commit}} 
		\FOR {time step $t \gets T_1 + T_2 + 1, \cdots, T$}
		\STATE Commit the batch of actions $\{\va_{t,m} \leftarrow \argmax_{\va \in \gA_m} \left\langle\va ,(\hatmB\hatmW)(m) \right\rangle : m \in [M]\}$.
		\ENDFOR 
	\end{algorithmic}
\end{algorithm}

Here, we present our main results for the multi-task setting. \cref{alg:etc} presents our algorithm, which can be divided into three stages.
We note that, at a high level, our algorithm is of the similar structure as the one in \citep{yang2021impact}. The main difference is the first stage where we use a different estimator

\textbf{Stage 1: Subspace Exploration.}
The goal of the first stage is to estimate the linear representation $\mB$.
To generate a good sampling distribution, we let each arm uniformly sample from a unit sphere.
For each task, we compute a rough estimation of $\vtheta_m$:  $\hatvtheta_m \gets \frac{d}{T_1} \sum_{t=1}^{T_1} r_{t, m} \va_{t, m}$.
Then we concatenate estimations for all task: $\hatmTheta =[\hatvtheta_1, \cdots, \hatvtheta_m] \in \sR^{d \times M}$, and use singular value decomposition on $\hatmTheta$ to estimate $\mB$.
The rational behind this step is that the true $\mTheta$ is a rank $k$ matrix with the left singular vectors being $\mB$.

Now we compare our method with the one in \citep{yang2021impact}.
Instead of computing SVD on $\hatmTheta$, they used SVD on  empirical weighted covariance matrix $
\hatmM = \frac{1}{T_1M} \sum_{t = 1}^{T_1} \sum_{m = 1}^M r_{n, t}^2  \va_{t, m}  \va_{t, m}^\top.
$
Note their squared version incur a larger error in estimating $\mB$ because they needed to first guarantee $\hatmM$ is close to $\mM$, and then use matrix perturbation analysis to argue the bound on $\mB$.
Our main improvement is based on the fact that  to achieve the same error rate, \emph{making $\hatmTheta$ close to $\mTheta$ requires less samples than making $\hatmM$ close to $\mM$}.
The reason is that the square operation in $\hatmM$ makes it concentrate slower to its mean when compared to our $\hatmTheta$. Furthermore, directly performing matrix perturbation analysis on $\hatmTheta$ gives tighter error bound than the squared estimator $\hatmTheta$.

\textbf{Stage 2: Per-Task Exploration.}
In this second stage, we aim to estimate $\mW$.
Since we already obtained a low-dimensional representation, we can just conduct exploration on the low-dimensional space.
Here we choose each column of $\hatmB$ as the exploration action.
After $T_2$ steps, we solve least square problem to estimate $\mW$.

\textbf{Stage 3: Commit.}
After the first two stages, we obtain accurate estimation $\left(\hatB\hatW\right)(m)$ for each task $m$.
For the remaining $(T-T_1 - T_2)$ steps, we just commit to the optimal induced by our estimations.. Specifically, we play the action $\va_{n,t} \gets \argmax_{\va \in \gA_t}\langle \va_t, \hatvtheta_t\rangle$ for step $t= T_1+T_2+1,\cdots,T$.

Next, we present the regret analysis of \cref{alg:etc}. 

\begin{thm} \label{thm:reg} The regret of \cref{alg:etc} is at most $\tildeO(d\sqrt{ k M T} + k M \sqrt{T })$. 
\end{thm}
This bound is near-optimal because it matches the lower bound in \citep{yang2021impact} up to logarithmic factors.
Since this is our main result, below we give its full proof.

\begin{proof}[Proof of Theorem~\ref{thm:reg}] 
	We first state how we decompose the regret.	
	Note that the regret of each task at each time step is bounded by $1$, so in \cref{alg:etc}, the regret incurred in stages 1, 2 is bounded by $M(T_1 + T_2)$. In stage 3, for the task $m$, the regret is bounded by 
	\begin{align}
		&\quad \E[\max_{\va \in \gA_m} \langle \va - \va_{t, m}, \vtheta_{m} \rangle] \notag \\
		&= 1 - \E[\langle \va_{t, m}, \vtheta_{m} \rangle] \\
		&\le \E\left[\frac{2 \norm{\hatvtheta_m - \vtheta_m}^2}{\norm{\vtheta_m}} \right] \label{eq:reg-800} \\ 
		&\le O\left(\E\left[\norm{\hatmB_\perp^\top \mB}^2 + \frac{k^2}{ T_2} \right]\right), \label{eq:theta_decomp}
	\end{align}
	where \eqref{eq:reg-800} uses Lemma \ref{lem:curvature} that leverages our assumption, and \eqref{eq:theta_decomp} uses Lemma~\ref{lem:theta_decomp}. Both lemmas have been proved in \citep{yang2021impact} and their proofs rely on routine linear algebra.
	Therefore, we can upper bound the regret by
	\begin{align}
		&\quad \E[\Reg^{M,T}] \notag \\
		&= \sum_{m = 1}^M \sum_{t=1}^T	\E[\max_{\va \in \gA_m} \langle \va - \va_{t, m}, \vtheta_{m} \rangle] \\
		&\le  M (T_1 + T_2) + M T \cdot  O\left(\E\left[\norm{\hatmB_\perp^\top \mB}^2 + \frac{k^2}{T_2} \right]\right) \label{eq:reg-900} \\
		&\le M T_1 + \tildeO\left(T \frac{d^2 k}{T_1}\right) + M T_2 + TM \frac{k^2}{T_2}  \label{eq:reg-1000} \\
		&\le \tildeO (d \sqrt{kMT} + Mk\sqrt{T}) \label{eq:reg-2000}
	\end{align}
	where \eqref{eq:reg-900} uses \eqref{eq:theta_decomp},  \eqref{eq:reg-1000} uses Lemma~\ref{lem:con1} below, and  \eqref{eq:reg-2000} uses our setting of $T_1$ and $T_2$.
	The main technical difficulty is in Lemma~\ref{lem:con1}, which we state and prove below.
\end{proof}

\begin{lem} \label{lem:con1} With probability $1 - O(M\delta)$, we have $\norm{\mTheta - \hatmTheta} \le O\left( \sqrt{\frac{d M \log\frac{d}{\delta}}{T_1}}\right)$ and $	\norm{\hat \mB^\top_\perp \mB} \le   O\left(\sqrt{\frac{d^2 k}{M T_1}}\right).$
\end{lem}
This lemma bounds our estimation error from stage 1 for learning the representation, and represents the main improvement over \citep{yang2021impact}.

\begin{proof} [Proof of Lemma~\ref{lem:con1}]
	Throughout the proof, we consider $t \in [T_1]$.	
	Define $\vz_{t, m} = d \va_{t, m}$. Then $\vz_{t, m} \sim \sS_d(d-1)$. We can write our estimation in the following form:
	\begin{align}
		\hatvtheta_m &= \frac{d}{T_1} \sum_{t=1}^{T_1} r_{t, m} \va_{t, m}\\
		&=  \frac{d}{T_1} \sum_{t=1}^{T_1} (\va_{t, m}^\top \vtheta_m + \eta_{t, m}) \va_{t, m} \\
		&=  \frac{1}{T_1} \sum_{t=1}^{T_1} (\vz_{t, m}^\top \vtheta_m + \sqrt d \eta_{t, m}) \vz_{t, m}.
	\end{align} 
	Therefore,  we can write the difference between our estimation and the truth as
	\begin{align}
		\hatvtheta_m - \vtheta_m &= \frac{1}{T_1} \sum_{t = 1}^{T_1} (\vz_{t, m} \vz_{t, m}^\top - \mI) \vtheta_m \notag \\
		&\qquad + \frac{\sqrt d}{T_1} \sum_{t = 1}^{T_1} \vz_{t, m} \eta_{t, m}. \label{eq:1-1}
	\end{align}
	By the standard matrix Bernstein inequality (cf. Lemma \ref{lem:matcon}) with parameter $v = O(d), b = O(1)$, we have with probability $1 - \delta$,
	\begin{align}
		\norm{\frac{1}{T_1}\sum_{t = 1}^{T_1} \vz_{t, m} \vz_{t, m}^\top - \mI} \le O\left( \sqrt{d\frac{\log\frac{d}{\delta}}{T_1}} \right).
	\end{align}
	Similarly, with probability $1 - \delta$,  we have 
	\begin{align}
		\norm{\frac{1}{T_1} \sum_{t = 1}^{T_1} \vz_{t, m} \eta_{t, m}} \le O\left( \sqrt{\frac{d\log\frac{d}{\delta}}{T_1}} \right).
	\end{align}
	Therefore, by \eqref{eq:1-1}, for each $m \in [M]$, we have that with probability $1 - 2\delta$, 
	\begin{align}
		\norm{\hatvtheta_m - \vtheta_m} \le  O( \sqrt{\frac{d\log\frac{d}{\delta}}{T_1}} ). \label{eq:con1-1}
	\end{align}
	Next, an application of Cauchy–Schwarz gives 
	\begin{align}
		\norm{\mTheta - \hatmTheta} \le \sqrt M \sqrt{\sum_{m = 1}^M \norm{\hatvtheta_m - \vtheta_m}},
	\end{align}
	and we conclude the bound on $\mTheta$ with a union bound over all tasks $m \in [M]$.

	To bound $\hatmB$, using a refined matrix perturbation analysis \citep[Remark 2]{cai2018rate}, we have 
	\begin{align}
		\norm{\hat \mB^\top_\perp \mB} &\le O(\frac{\norm{\mTheta - \hatmTheta}}{\sigma_k(\mTheta)}) \notag \\
		&\le  O\left(\frac{d / \sqrt{T_1}}{\sigma_k\left(\mTheta\right)}\right) =  O\left(\sqrt{\frac{d^2 k}{M T_1}}\right).
	\end{align}
\end{proof}

\section{Main Results for Lifelong Linear Bandits with Shared Representation}
\label{sec:life_long}
In this section we describe our algorithm for lifelong linear bandits and its theoretical guarantees.

This algorithm has a task-specific exploration stage and a representation learning stage, which are similar in Algorithm~\ref{alg:etc}.
Suppose we are in task $m$. 
In the \textbf{Task-Specific Exploration} stage, we operate on an estimated linear representation $\hatmB$, choose its column as action, each for $n_{m,1}$ times.
Then we estimate the $j$-th entry of $\vw_m$ simply by averaging the reward explored by the $j$-th column of $\hatmB$.
If $\hatmB$ is accurate, then this estimator for $\vw_m$ will be accurate as well.
Then, we can recover $\vtheta_m$ by $\tildevtheta_m = \hatmB_{m-1} \tildevw_m$.

However, our estimation of $\mB$ can be inaccurate, e.g., in the beginning, and we need to design a criterion to detect whether our estimation is accurate enough.
We show that  checking whether the norm of $\tildevtheta_m$ is smaller than $1-\epsilon$ is an effective criterion: if this holds, then it means there exists direction in the linear representation that we have not explored, and this direction matters for task $m$.
In this case, we will enter the next stage, \textbf{Re-estimate Linear Representation}.

In this stage, we choose actions to be the standard basis vectors and then obtain a rough estimation of $\vtheta_{m}$.
In order to find the missing direction, we project our the rough estimation on the  complement space of $\hatmB$ (cf.~\eqref{eq:b_hat_tau}). We will show this is an accurate estimator for the missing direction in the underlying linear representation.
Then, we add this new direction to our linear representation (cf.~\eqref{eq:alg-lll-1}).

These two steps, together with the correct choices of the number of samples to explore ($n_{m,1}, n_{m,2}$ in Algorithm~\ref{alg:ll}), will give us the pure exploration guarantee stated below.

\begin{thm}[Pure Exploration Bound of Algorithm~\ref{alg:ll}] \label{thm:lll-pe} With probability $1 - O(dM\delta)$, \cref{alg:ll} outputs near-optimal $\hatvtheta_m$ such that $\norm{\hatvtheta_{m} - \vtheta_{m}} \le \epsilon$ for every task $m \in [M]$, using at most $\tilde O(\frac{d^2 k + k^2 M }{\epsilon^2})$ samples.
\end{thm}

Now we consider the regret guarantee. Recall that each task has in total $T$ steps.
Here we choose $\epsilon^2 = \tildeO\left(\sqrt{\frac{d^2k+k^2M}{ MT}}\right)$.
After we explored in stage 1 and stage 2, we just use all remaining steps to commit to the greedy action with respect to the estimated $\hatvtheta_m$.
The following theorem states our regret bound.

\begin{thm}[Regret of Algorithm~\ref{alg:ll}] \label{thm:lllreg} Algorithm~\ref{alg:ll} incurs at most $\tildeO(d\sqrt{kMT} + kM \sqrt T)$ regret.
\end{thm}
Note this bound has the same form as in the multi-task setting. Also, it straightforward to see the lower bound construction in \citep{yang2021impact} also applies to the lifelong learning setting. Therefore, our regret bound for the lifelong learning setting is again near-optimal.

\begin{algorithm}[t] 
	\caption{\textbf{L}ife\textbf{l}ong \textbf{L}inear Bandits (LLL)}\label{alg:ll}
	\begin{algorithmic}[1]
		\STATE {\bf \underline{Input:} } Dimension $d$, rank $k$, number of tasks $M$, accuracy $\epsilon$.
		\STATE {\bf \underline{Initialize:} } {Set $\tau_0 \gets 0, \hatmB_0   \in \sR^{d \times \tau_0}$}.
		\FOR {task $m \gets 1$ to $M$}
		\STATE {\bf \underline{Stage 1: Task-Specific Exploration}}
		\FOR {$j \gets 1$ to $\tau_{m-1}$}
		\STATE Choose $\va_{t, m} = \hatmB_{m-1}(j)$ for $t = (j-1) n_{m,1} + 1, \cdots, j n_{m, 1}$, where $n_{m, 1} = \lceil \frac{4\tau_{m-1} \log \frac2\delta}{\epsilon^2} \rceil$.
		\STATE Compute estimate $\tildevw_m(j) = \frac{1}{n_{m, 1}} \sum_{t=(j-1) n_{m,1} + 1}^{j n_{m, 1}} r_{t, m}$
		\ENDFOR
		\STATE Let $\tildevtheta_m = \hatmB_{m-1} \tildevw_m$
		\IF{$\norm{\tildevtheta_m} \le 1-\epsilon$}
		\STATE {\bf \underline{Stage 2 Re-estimate Linear Representation}}
		\FOR {dimension $j \gets 1$ to $d$}
		\STATE Choose $\va_{t, m} = \ve_j$ for $n_{m, 2} = \lceil \frac{16d \log\frac2\delta}{\epsilon^2} \rceil$
		\STATE Compute estimate $\hatvtheta_m(j) = \frac{1}{n_{m, 2}} \sum_t r_{t, m}$
		\ENDFOR
		\STATE Update $\tau_m = \tau_{m-1} + 1, \nu_{\tau_m} = m$ and 
		\begin{align}
			\hatvb_{\tau_m} &= \frac{\mP_{\hatmB_{m-1}}^\perp \hatvtheta_m}{\norm{\mP_{\hatmB_{m-1}}^\perp \hatvtheta_m}}, \label{eq:b_hat_tau} \\
			\hatmB_m &= [\hatmB_{m-1}, \hatvb_{\tau_m}] \in \sR^{d \times \tau_m} \label{eq:alg-lll-1}
		\end{align}
		\ELSE
		\STATE Update $\tau_m = \tau_{m-1}, \hatmB_m = \hatmB_{m-1}, \hatvtheta_m = \tildevtheta_m$.
		\ENDIF
		\STATE {\bf For Regret Minimization:} Choose action $\va_{t,m} = \argmax_{\va \in \gA_m} \langle \va, \hatvtheta_m \rangle$ until reaching horizon $T$.
		\ENDFOR
	\end{algorithmic}
\end{algorithm}

\subsection{Proof of Theorem~\ref{thm:lll-pe}}
In this section we give the proof of Theorem~\ref{thm:lll-pe}.
The Theorem~\ref{thm:lllreg} is a straightforward corollary of Theorem~\ref{thm:lll-pe}, so we defer the proof to appendix.

\begin{proof}[Proof of Theorem~\ref{thm:lll-pe}] First, we show that $\norm{\hatvtheta_{m} - \vtheta_{m}} \le \epsilon$ when Lemmas \ref{lem:lll-con1}, \ref{lem:lll-con-2}, and \ref{lem:lll-con-3} hold. For each $m \in [M]$, if the algorithm does not enter stage 2, we have $\norm{\tildevtheta_m} \ge 1-\epsilon$, so we have
	\begin{align}
		\norm{\vtheta_{m} - \tildevtheta_m} \le \norm{\vtheta_{m}} - \norm{\tildevtheta_m} \le \epsilon.
	\end{align} 
	If the algorithm enters stage 2, then we have $\norm{\vtheta_{m} - \tildevtheta_m} \le \epsilon$ by Lemma \ref{lem:lll-con-2} and plugging in the value of $n_{m, 2}$. 
	
	Second, we analyze the sample complexity of the algorithm. Task $m$ uses $\tau_m n_{m, 1}$ samples in stage 1. If task $m$ enters stage 2, it uses additional $d n_{m,2}$ samples. There are only $\tau_M$ tasks entering stage 2, so the sample complexity is upper bounded by 
	\begin{align}
		&\quad \sum_{m = 1}^M \tau_{m} n_{m, 1} + \tau_M d n_{m, 2} \notag \\
		&\le O\left(M \frac{\tau_M^2 \log\frac1\delta}{\epsilon^2} + \tau_M \frac{d^2 \log\frac1\delta}{\epsilon^2}\right) \\ 
		&\le \tildeO\left(\frac{k^2M + d^2 k}{\epsilon^2}\right), \label{eq:lll-final-1}
	\end{align}
	where (\ref{eq:lll-final-1}) uses \cref{lem:lll-4}.
\end{proof}
In the following, we state the key lemmas. The proofs of Lemma ~\ref{lem:lll-con1},~\ref{lem:lll-2}, and~\ref{lem:lll-con-2} are deferred to appendix.

The first lemma states that with high probability, our estimation in stage 1 is close to the truth projected to our estimated linear representation.
\begin{lem} Define the event $\gE_1 = \{\forall m \in [M] : \norm{\tildevtheta_m - \mP_{\hatmB_{m-1}}\vtheta_m}\le \sqrt{\frac{\tau_{m-1}\log\frac{2dM}{\delta}}{n_{m,1}}}\}$. Then $\Pr[\gE_1] \ge 1 - \delta$. \label{lem:lll-con1}
\end{lem}

The next lemma justifies our criterion on whether to enter stage 2. 
\begin{lem} \label{lem:lll-2} If $\gE_1$ (defined in \cref{lem:lll-con1}) holds  and \cref{alg:ll} enters Stage 2 (i.e., $\norm{\tildevtheta_m} \le 1-\epsilon$), then  $\norm{\mP_{\hatmB_{m-1}}^\perp \vtheta_{m}} \ge \epsilon$.
\end{lem}

Let $\vbeta_{\tau_m} = \mP_{\mB} \hatvtheta_{m}$ the estimation projected on the true representation,$ \vd_{\tau_m} = \hatvtheta_m - \vbeta_{\tau_m}$ the difference and define $\hatvw_\tau$ by $\mB \hatvw_\tau = \vbeta_\tau$. 
We have the following lemma that characterizes our estimator $\hatvtheta$.

\begin{lem} \label{lem:lll-con-2} Define the event $\gE_2 = \{\forall \ell \le M \text{ and } m = \nu_\ell \le M : \max\{\norm{\vtheta_m - \vbeta_{\ell}}, \norm{\vd_{\ell}}\} \le \norm{\vtheta_{m} - \hatvtheta_{m}} \le \sqrt{\frac{d \log \frac{2dM}{\delta}}{n_{m,2}}}\}$ where $\nu_\ell$ is defined in Algorithm~\ref{alg:ll}. Then $\Pr[\gE_2] \ge 1 - \delta$.
\end{lem}

Lastly, we use the following lemma to bound number of times we enter stage 2.
This is a key lemma and is different from the previous paper on lifelong learning~\citep{cao2021provable}, so we give the full proof here.
\begin{lem} $\tau_M \le O(k \log \frac{k}{\epsilon})$. \label{lem:lll-4}
\end{lem}
\begin{proof}
	We first obtain a bound on $\mP^\perp_{\hatmB_{m-1}} \hatvb_{\tau}$:
	\begin{align}
		&\quad \norm{\mP^\perp_{\hatmB_{m-1}} \hatvb_{\tau}} \notag \\
		&\ge \norm{\mP^\perp_{\hatmB_{m-1}} \vtheta_m} - \norm{\mP^\perp_{\hatmB_{m-1}}(\vtheta_m - \hatvb_\tau)} \\ 
		&\ge \norm{\mP^\perp_{\hatmB_{m-1}} \vtheta_m} - \norm{\vtheta_m - \hatvb_\tau} \\
		&\ge \frac\epsilon4,\label{eq:lll-2-2-1} 
	\end{align}
	where \eqref{eq:lll-2-2-1} is by Lemmas \ref{lem:lll-2}, \ref{lem:lll-con-2} and by plugging in $n_{m,2}$. 
	and thus we have 
	\begin{align}
		\abs{\hatvb_{\tau_m}^\top \mP_{\hatmB_{m-1}}^\perp \hatvb_{\tau_m}} \ge \frac{\epsilon^2}{16}. \label{eq:lll-1}
	\end{align}
	Now, let $\tau = \tau_m, \hatmW_{\tau_m} = [\hatvw_1, \cdots, \hatvw_{\tau_m}]$ and  define $\mD_{\tau} = [\vd_1, \cdots, \vd_{\tau}]$. 
	By \cref{lem:lll-con-2}, we have $\norm{\vd_{\tau_m}} \le \epsilon$. By \cref{lem:lll-con-3}, we have  $\mD_\tau^\top \mD_\tau \preccurlyeq \epsilon^2 \mI$ whenever $\tau \le O(d \log\frac{d}{\delta})$. 
	Below we use a volumetric argument to bound the growth of $\tau_m$.
	We note that this is a different argument from the one in \citep{cao2021provable}. Here we mainly rely on the standard elliptical potential lemma whereas they used a polyhedral volumetric argument.
	Specifically, starting from \eqref{eq:lll-1}, we have 
	\begin{align}
		\frac{\epsilon^2}{16} &\le 	\abs{\hatvb_{\tau_m}^\top \mP_{\hatmB_{m-1}}^\perp \hatvb_{\tau_m}} \\
		&= \hatvb_{\tau}^\top \Big(\mI - (\mB \hatmW_{\tau-1} + \mD_{\tau-1})\notag \\
		&\qquad \quad \quad (\hatmW_{\tau-1}^\top \mW_{\tau-1} + \mD_{\tau-1}^\top \mD_{\tau-1})^{-1} \notag \\
		&\qquad \qquad (\mB \mW_{\tau-1} + \mD_{\tau-1})^\top \Big) \hatvb_{\tau} \notag \\
		&= \hatvw_{\tau}^\top \Big(\mI - \hatmW_{\tau-1}(\hatmW_{\tau-1}^\top \hatmW_{\tau-1} + \mD_{\tau-1}^\top \mD_{\tau-1})^{-1} \notag \\
		&\qquad \qquad  \hatmW_{\tau-1}^\top \Big) \hatvw_{\tau} \\ 
		&\le \hatvw_{\tau}^\top \Big(\mI - \hatmW_{\tau-1}(\hatmW_{\tau-1}^\top \hatmW_{\tau-1} + \epsilon^2  \mI)^{-1} \notag \\
		&\qquad \qquad \mW_{\tau-1}^\top\Big) \vw_{\tau} \\ 
		&= \epsilon^2 \norm{\hatvw_{\tau}}_{(\mV_{\tau-1} + \epsilon^2\mI)^{-1}}. 
	\end{align}
	
	Therefore, we have
	\begin{align}
		\hatvw_\tau^\top (\mV_{\tau} + \epsilon^2\mI)^{-1} \hatvw_\tau \ge \frac{1}{16}, 
	\end{align}
	and thus we obtain
	\begin{align}
		\hatvw_\tau^\top (\mV_{\tau-1} + \epsilon^2\mI)^{-1} \hatvw_\tau \ge \frac{1}{32}, 
	\end{align}
	
	By Lemma \ref{lem:ellip}, we have 
	\begin{align}
		\frac{\tau_M}{32} &\le \sum_{\tau = 1}^{\tau_M} \norm{\hatvw_\tau}_{(\mV_{\tau-1} + \epsilon^2 \mI)^{-1}}  \\
		&\le O(\sqrt{\tau_M k \log \frac{\tau_M + k \epsilon^2}{k \epsilon^2}}),
	\end{align}
	which implies
	$
	\tau_M \le O(k \log \frac{k}{\epsilon}).
	$
\end{proof}

\section{Experiments}
\label{sec:exp}
\begin{figure*}[!t]
	\centering
	\begin{minipage}{0.33\textwidth}
		\centering
		\includegraphics[width=1.05\textwidth]{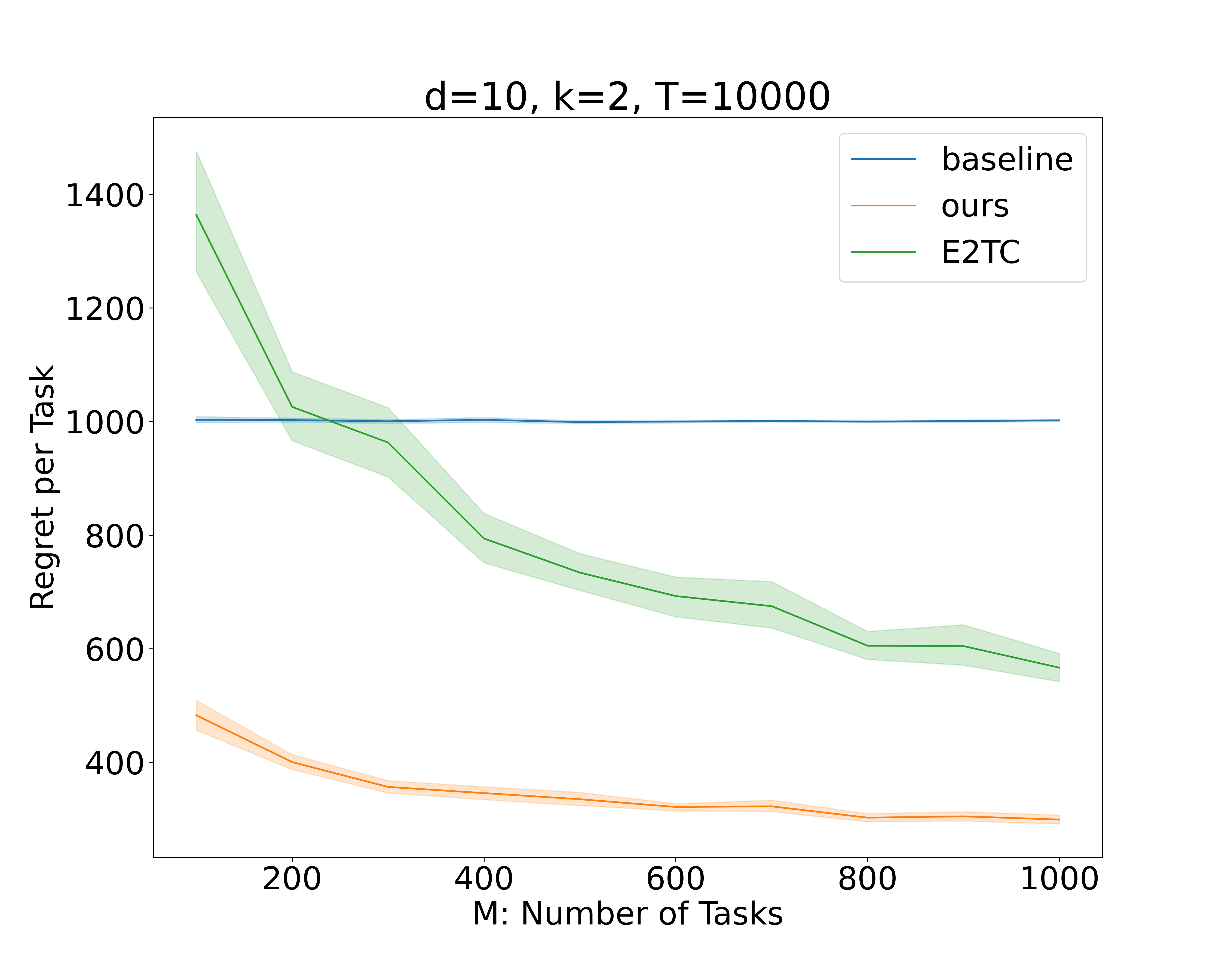}
	\end{minipage}
	\hspace{-1em}
	\begin{minipage}{0.33\textwidth}
		\centering
		\includegraphics[width=1.05\textwidth]{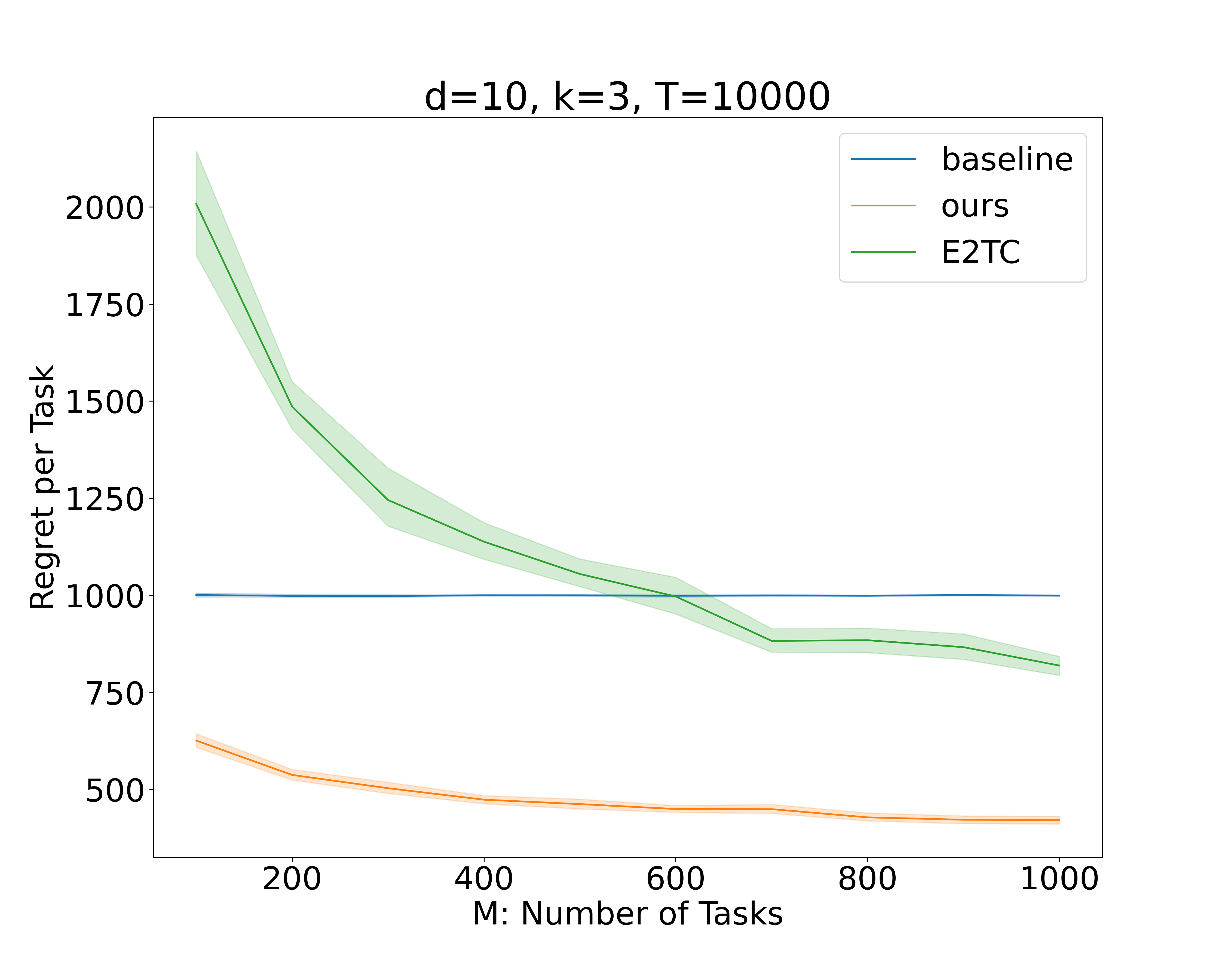}
	\end{minipage}
	\hspace{-1em}
	\begin{minipage}{0.33\textwidth}
		\centering
		\includegraphics[width=1.05\textwidth]{figs/multitask3.png}
	\end{minipage}
	\caption{
		\label{fig:multi}
		Comparisons of the E2TC algorithm in \citep{yang2021impact} and our \cref{alg:etc} for $k \in \{2, 3, 4\}$.}
\end{figure*}

\begin{figure*}[!t]
	\centering
	\begin{minipage}{0.33\textwidth}
		\centering
		\includegraphics[width=1.05\textwidth]{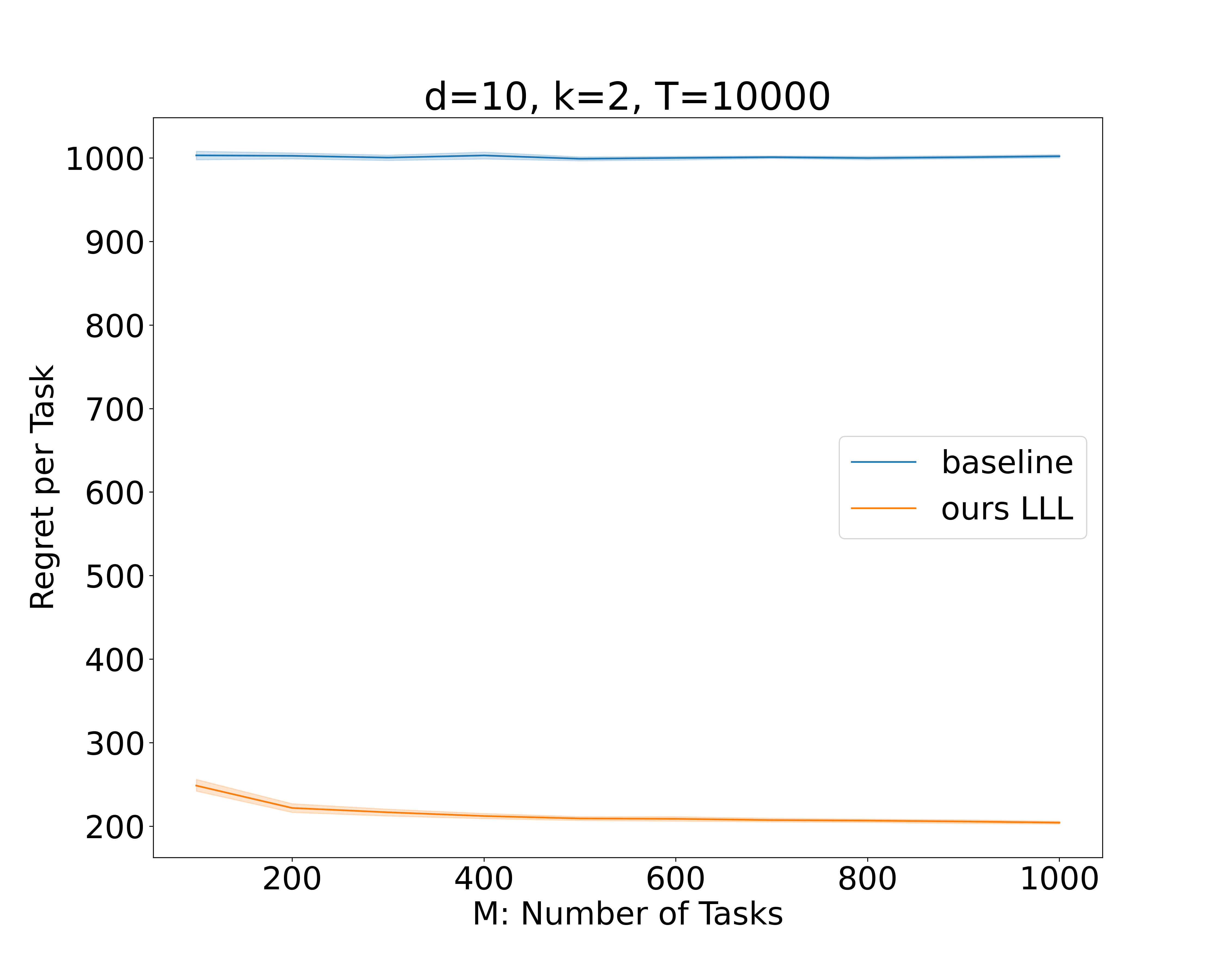}
	\end{minipage}
	\hspace{-1em}
	\begin{minipage}{0.33\textwidth}
		\centering
		\includegraphics[width=1.05\textwidth]{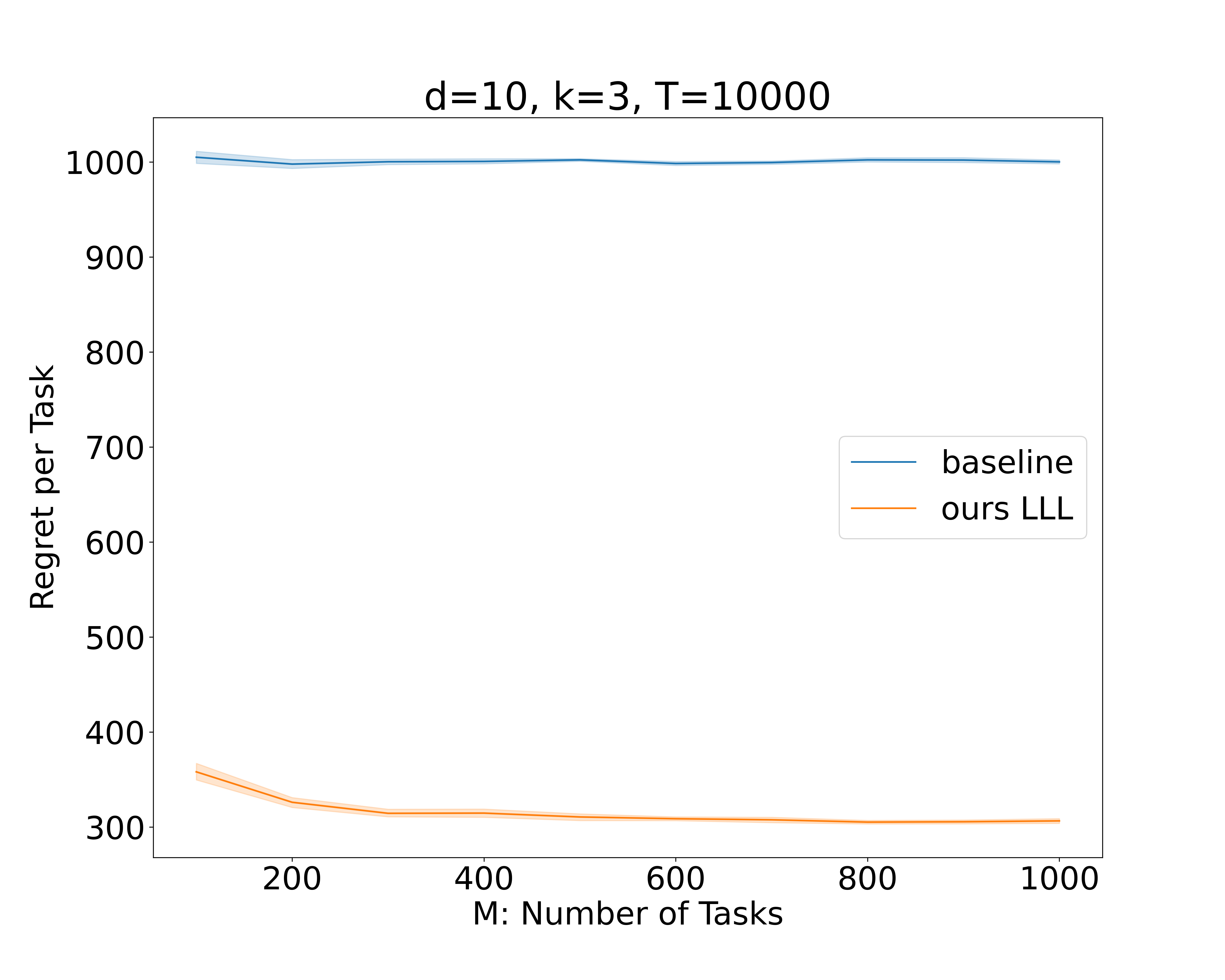}
	\end{minipage}
	\hspace{-1em}
	\begin{minipage}{0.33\textwidth}
		\centering
		\includegraphics[width=1.05\textwidth]{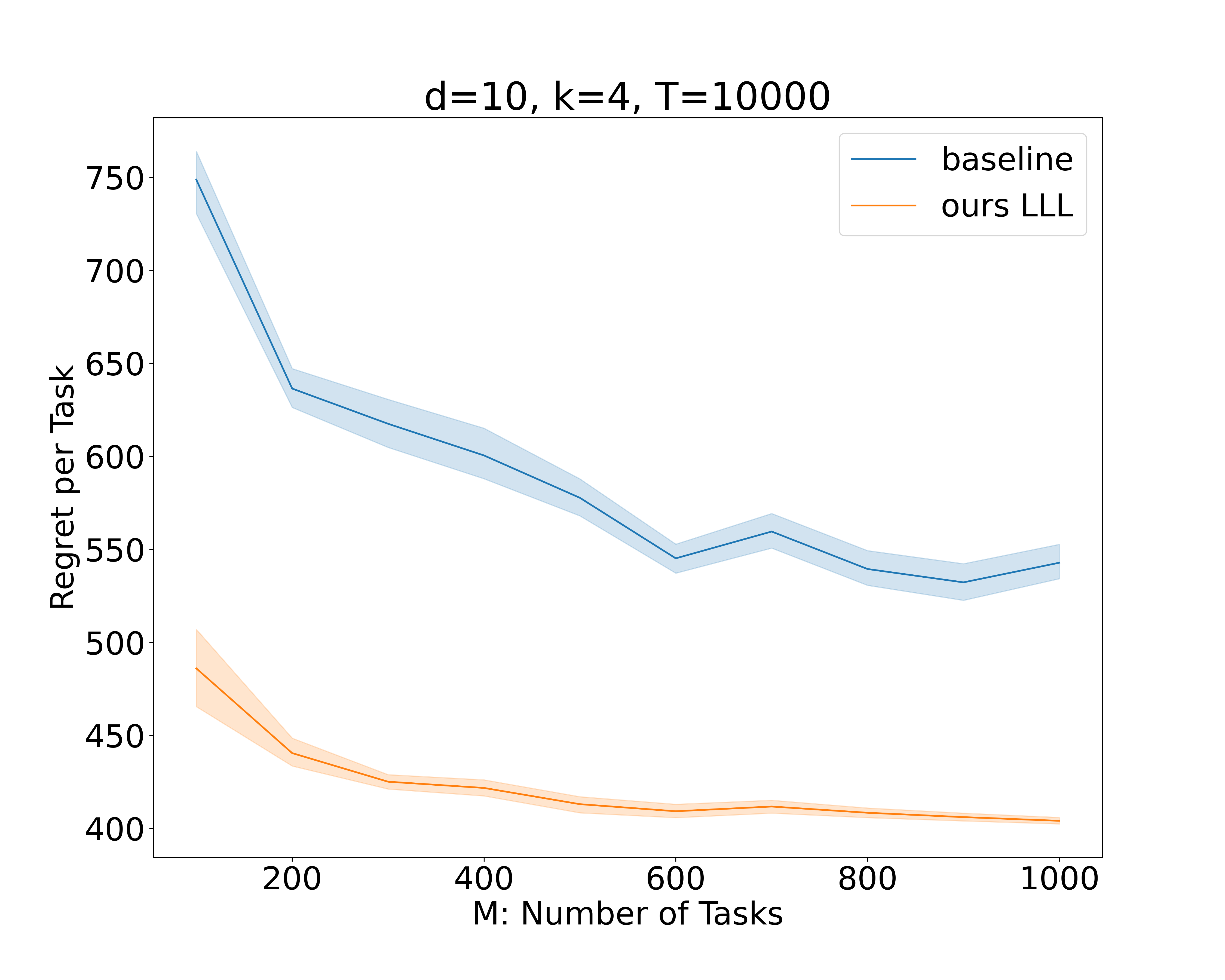}
	\end{minipage}
	\caption{
		\label{fig:lll-1}
		Comparisons of the naive baseline algorithm  and our \cref{alg:ll} for $k \in \{2, 3, 4\}$.}
\end{figure*}
\begin{figure*}[!t]
	\begin{minipage}{0.33\textwidth}
		\centering
		\includegraphics[width=1.05\textwidth]{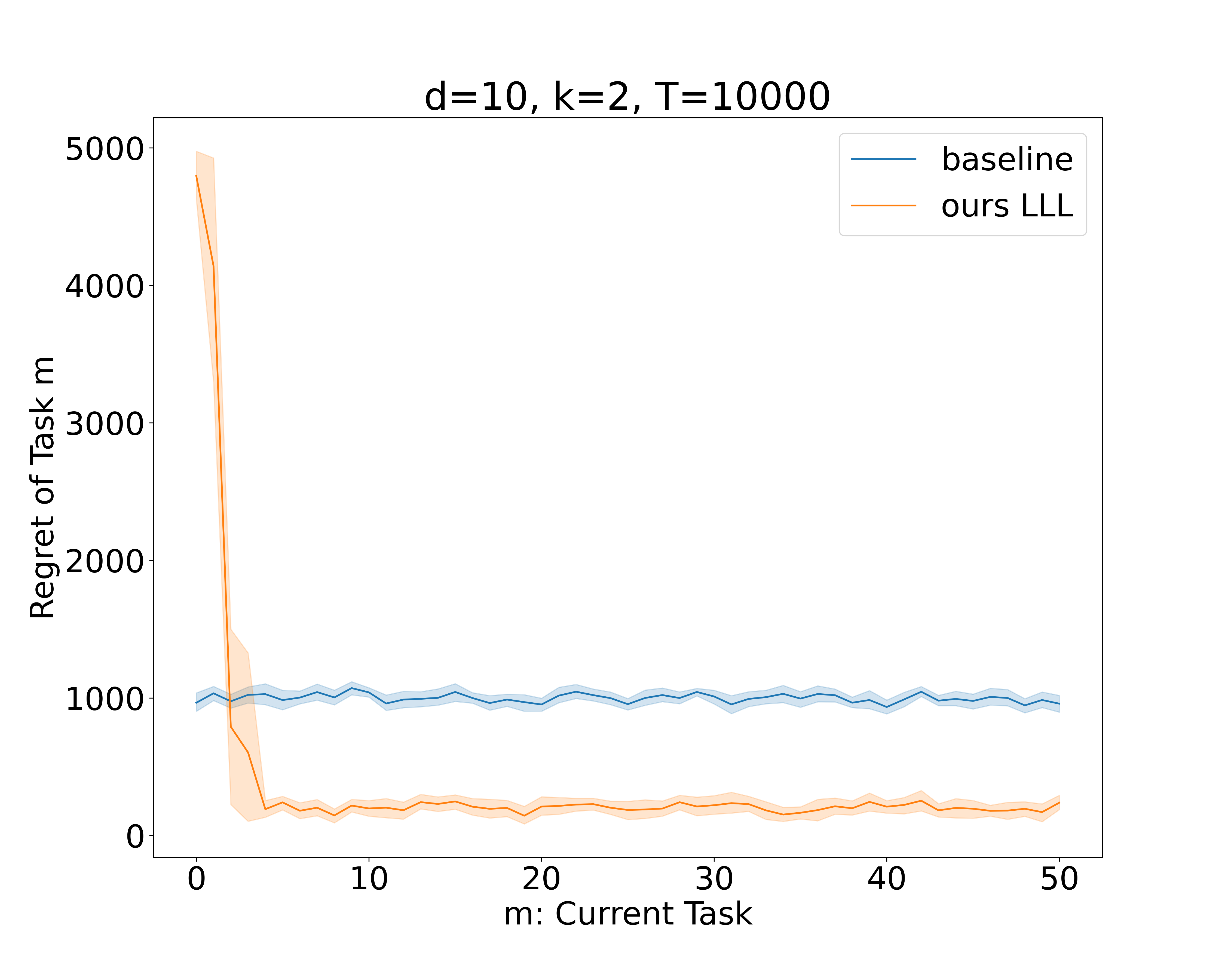}
	\end{minipage}
	\hspace{-1em}
	\begin{minipage}{0.33\textwidth}
		\centering
		\includegraphics[width=1.05\textwidth]{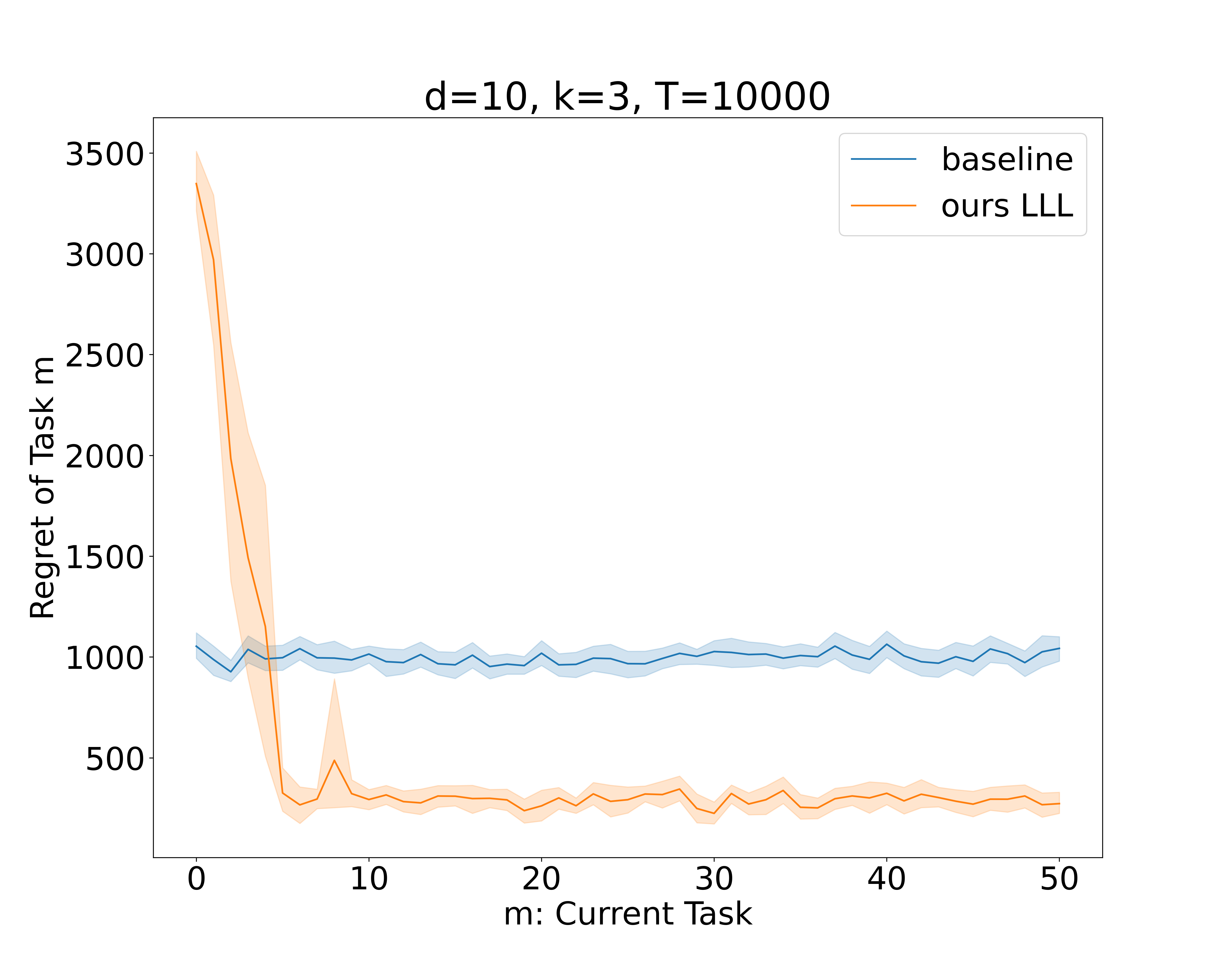}
	\end{minipage}
	\hspace{-1em}
	\begin{minipage}{0.33\textwidth}
		\centering
		\includegraphics[width=1.05\textwidth]{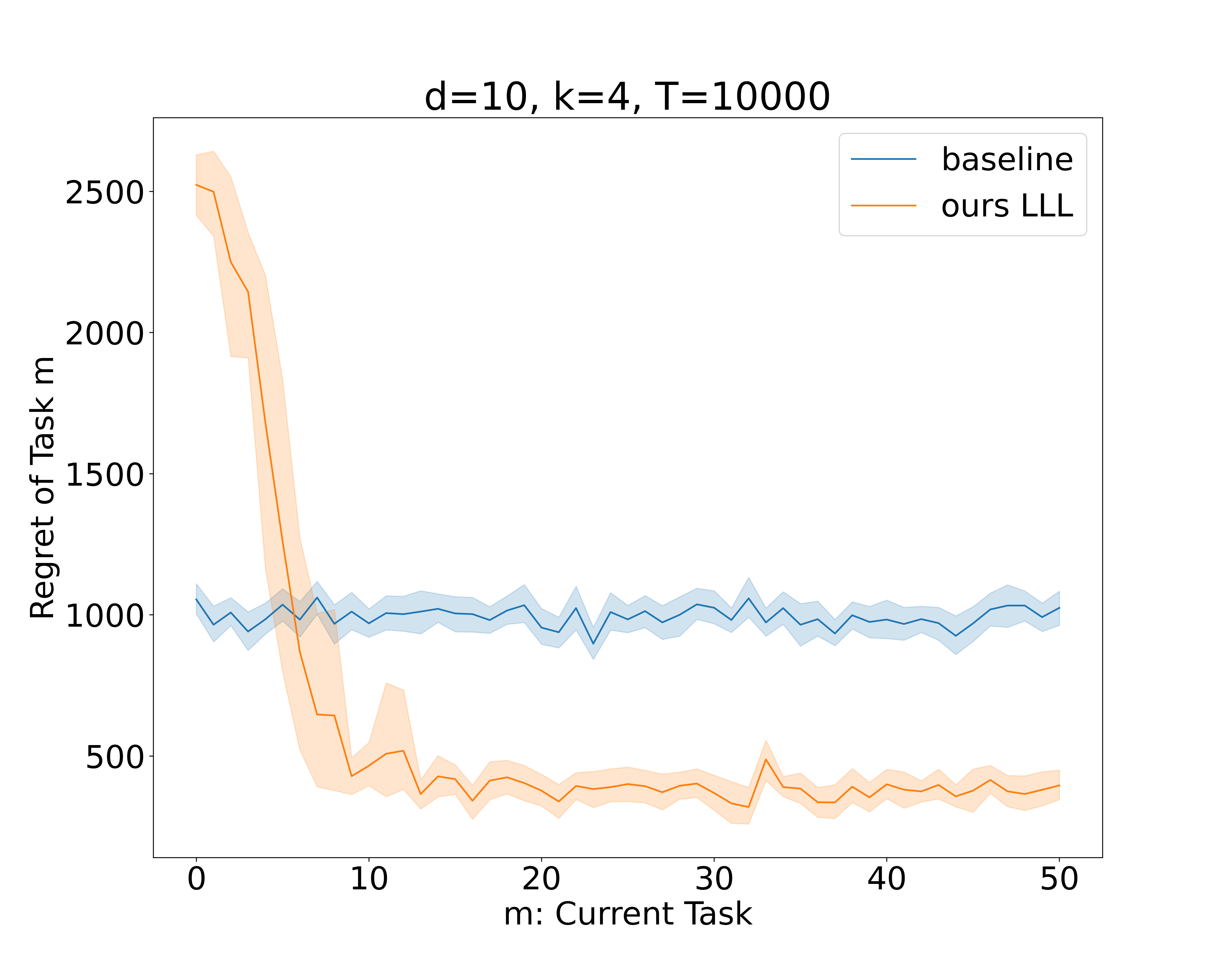}
	\end{minipage}
	\caption{Comparisons of the naive baseline algorithm and our \cref{alg:ll} for $k \in \{2, 3, 4\}$.
		\label{fig:lll-2}
	}
\end{figure*}

\paragraph{Setup.} We run experiments on synthetic data. Following \citep{yang2021impact}, we set $d=10, T=10^4$ and we consider $k = 2, 3, 4$ in our experiments. $\mB$ is generated by taking the first $k$ columns of a uniformly random orthonormal matrix, and each $\vw_m$ is generated uniformly from sphere. 

\paragraph{Results for the Multi-Task Setting.} We compare our algorithm with the baseline algorithm which treats each task independently, and the E2TC algorithm in \citep{yang2021impact}. The results are shown in Figure~\ref{fig:multi}.
For all settings, our algorithm achieves smaller regret, demonstrating the effectiveness of our new estimator.

\paragraph{Results for the Lifelong Setting.}  Since lifelong linear bandits is a new setting, we only compare our algorithm, LLL, with the naive baseline algorithm which  treats each task independently. In Figure~\ref{fig:lll-1}, we vary the number of total tasks and report the per task regret. We observe that our algorithm significantly outperforms the base and its per task regret decreases as we increase the number of tasks, which matches our theory. In Figure~\ref{fig:lll-2}, we report the regret on every task $m$ for in total $50$ tasks. We observe that for the initial tasks, we incur a higher regret than the baseline. However, the regret for later tasks is significantly smaller than the baseline. This  matches our theory which indicates that after learning a good low-dimensional representation in the beginning, it benefits the efficiency in later tasks and reduces the regret.

\section{Conclusion}
\label{sec:con}
In this paper, we gave near-optimal algorithms for multi-task and lifelong linear bandits with shared representation. 
A natural future direction is extend our algorithm to general representations, such as deep neural networks.
Another interesting direction is to design optimal algorithms for  multi-task reinforcement learning with shared representation.

\section*{Acknowledgments}

QL is supported by NSF 2030859 and the Computing Research Association for the CIFellows Project. JDL acknowledges support of the ARO under MURI Award W911NF-11-1-0304,  the Sloan Research Fellowship, NSF CCF 2002272, NSF IIS 2107304,  and an ONR Young Investigator Award. Simon S. Du gratefully acknowledges funding from NSF Award’s IIS-2110170 and DMS-2134106.

\input{main.bbl}
\bibliographystyle{plainnat}

\newpage
\appendix
\section{Technical Lemmas and Omitted Proofs}
The following matrix Bernstein inequality is a straightforward corollary of  Theorem 6.1.1 in \citep{tropp2015introduction}.

\begin{lem} \label{lem:matcon} Let $\mA_1, \cdots, \mA_n \in \sR^{d \times k}$ be a sequence of i.i.d. random matrices, such that $\E \mA_i = 0$ and $\norm{\mA_i} \le b$ for every $i \in [n]$. Let $\norm{\cdot}$ be the spectral norm of matrices. Let 
	\begin{align}
		v = \max_i \{\norm{ \E[\mA_i \mA_i^\top]}, \norm{\E[\mA_i^\top \mA_i]}\},
	\end{align}
	then  with probability $1 -O(\delta)$,
	\begin{align}
		\norm{\frac1n \sum_{i = 1}^n \mA_i } \le O(\sqrt{\frac{v \log \frac{d+k}{\delta}}{n}} + \frac{b\log \frac{d+k}{\delta}}{n}).
	\end{align}
\end{lem}

\begin{lem}[\citet{hoeffding1963probability}] \label{lem:con} Let $X_1, \cdots, X_n$ be a sequence of i.i.d. random variables, such that $\E X_i = 0$ and $\abs{X_i} \le b$ for every $i \in [n]$. Then  with probability $1 -\delta$,
	\begin{align}
		\abs{\frac1n \sum_{i = 1}^n X_i } \le b\sqrt{\frac{\log \frac{2}{\delta}}{n}}.
	\end{align}
\end{lem}
\begin{lem}[\citet{yang2021impact}, Lemma 17] \label{lem:curvature} Under Assumption ?, we have 
	\begin{align}
		1 - \langle \va_m, \vtheta_m\rangle \le \frac{\norm{\va_m - \vtheta_m}_2^2}{\norm{\vtheta_m}}.
	\end{align}
\end{lem}

\begin{lem}[\citet{yang2021impact}, Lemma 18 applied to our setting] \label{lem:theta_decomp}$\E\norm{\hatvtheta_m - \vtheta_m}^2 \le O\left(\frac{k^2}{T_2}\right) + \norm{\hatmB^\top_\perp \mB}^2$.
\end{lem}

\begin{lem}[Elliptical potential lemma, \citet{abbasi2011improved}, Lemma 11] \label{lem:ellip} Let $\vx_1, \cdots, \vx_n \in \sR^d$ be a sequence of vectors and let $\mLambda_i = \lambda \mI + \sum_{j = 1}^i \vx_j \vx_j^\top$. Assume $\norm{\vx_i} \le \lambda$. Then 
	\begin{align}
		\sum_{i = 1}^n \min\{1, \norm{\vx_i}_{\mLambda_{i-1}^{-1}}^2\} \le 2 \ln \frac{\det \mLambda_n}{\det \mLambda_0} \le 2 d\ln \frac{\tr \mLambda_n + n}{\lambda}.
	\end{align}
\end{lem}

\begin{proof}[Proof of Theorem~\ref{thm:lllreg}] By \cref{thm:lll-pe}, the regret can be bounded by
	\begin{align}
		&\quad \tildeO(\frac{d^2 k + k^2 M}{\epsilon^2}) + \sum_{m = 1}^M \sum_{t=1}^T \max_{\va \in \gA} \langle \va - \va_{t, m}, \vtheta_m \rangle \notag \\ 
		& \le \tildeO(\frac{d^2 k + k^2 M}{\epsilon^2}) + \sum_{m = 1}^M T \cdot O(\norm{\hatvtheta_m - \vtheta_m}^2) \\ 
		&\le \tildeO(\frac{d^2 k + k^2 M}{\epsilon^2}) + O(MT \epsilon^2) \\ 
		&\le \tildeO(d\sqrt{kMT} + kM \sqrt T ), \label{eq:lll-reg-2}
	\end{align}
	where (\ref{eq:lll-reg-2}) is by choosing the optimal $\epsilon^2 = \tildeO(\sqrt{(d^2k+k^2M)/ MT})$.
\end{proof}

\begin{proof}[Proof of Lemma~\ref{lem:lll-con1}] First, we note that $\hatmB_m$ always have orthogonal columns by induction, because by \eqref{eq:alg-lll-1}, the last column $\hatvb_{\tau_m} \propto \mP^\perp_{\hatmB_{m-1}}\hatvtheta_m$ is always orthogonal to the first $\tau_m -1 $ columns, which is $\hatmB_{m-1}$.

	Next, note that for $t \in [(j-1) n_{m,1} + 1, j n_{m, 1}]$, we have $\E[r_{t, m}] = \langle \hatmB_{m-1}(j), \vtheta_m \rangle$. Therefore, conditioned on $\tau_{m-1}$, by \cref{lem:con}, for each $j \in [\tau_{m-1}]$, with probability $1 - \frac{\delta}{dM}$, 
	\begin{align}
		\abs{\tildevw_m(j) - \langle \hatmB_{m-1}(j), \vtheta_m \rangle} \le \sqrt{\frac{\log \frac{2dM}{\delta}}{n_{m, 1}}}. \label{eq:lll-1-1}
	\end{align}
	By a union bound over $j \in [\tau_{m-1}]$, with probability $1 - \frac{\delta}{M}$, (\ref{eq:lll-1-1}) holds for every $j \in [\tau_{m-1}]$. Note that $\langle \tildevtheta_m, \hatmB_{m-1}(j) \rangle = \tildevw_m(j)$, so with probability $1 - \frac{\delta}{M}$, 
	\begin{align}
		&\quad \norm{\tildevtheta_m - \mP_{\hatmB_{m-1}}\vtheta_m} \notag \\
		&= \sqrt{\sum_{j = 1}^{\tau_{m-1}} \langle \hatmB_{m-1}(j), \tildevtheta_m - \vtheta_{m} \rangle^2 } \label{eq:lll-1-1-1} \\ 
		&\le \sqrt{\sum_{j = 1}^{\tau_{m-1}} \frac{\log \frac{2dM}{\delta}}{n_{m, 1}}} \label{eq:lll-1-2} \\
		&\le \sqrt{\frac{\tau_{m-1}\log\frac{2dM}{\delta}}{n_{m,1}}}, \label{eq:lll-1-3}
	\end{align} 
	where (\ref{eq:lll-1-1-1}) uses that $\hatmB_{m-1}$ has orthogonal columns. Finally, we conclude by a union bound over $m \in [M]$. 
\end{proof}

\begin{proof}[Proof of Lemma~\ref{lem:lll-2}] We have
	\begin{align}
		&\quad \norm{\mP_{\hatmB_{m-1}}^\perp \vtheta_{m}} \notag \\
		&= \norm{\vtheta_{m} - \mP_{\hatmB_{m-1}} \vtheta_{m}} \\ 
		&\ge \norm{\vtheta_{m}} - \norm{\mP_{\hatmB_{m-1}} \vtheta_{m}} \\ 
		&\ge 1 - (\norm{\tildevtheta_m} + \norm{\tildevtheta_m - \mP_{\hatmB_{m-1}}\vtheta_m}) \label{eq:lll-2-1-1} \\ 
		&\ge \frac\epsilon2, \label{eq:lll-2-1-2}
	\end{align}
	where \eqref{eq:lll-2-1-2} is by \cref{lem:lll-con1} and by plugging in $n_{m,1}$.
\end{proof}

\begin{proof}[Proof of Lemma~\ref{lem:lll-con-2}] Note that $\nu_\ell = m$ and $\tau_m = \ell$. By \cref{lem:con}, for each $j \in [d]$, condition on $\nu_\ell = m$, with probability $1 - \frac{\delta}{dM}$, 
	\begin{align}
		\abs{ \hatvtheta_m(j) - \vtheta_m(j)} \le \sqrt{\frac{\log \frac{2dM}{\delta}}{n_{m, 2}}}. \label{eq:lll-2-1}
	\end{align}
	By a union bound over $j \in [d]$, with probability $1 - \frac{\delta}{M}$, (\ref{eq:lll-2-1}) holds for every $j \in [d]$. Then  
	\begin{align}
		\norm{\vtheta_m - \hatvtheta_m} &= \sqrt{\sum_{j = 1}^{d} (\hatvtheta_m(j) - \vtheta_m(j))^2 } \\ 
		&\le \sqrt{\sum_{j = 1}^{d} \frac{\log \frac{2dM}{\delta}}{n_{m, 2}}}  \label{eq:lll-2-2} \\
		&= \sqrt{\frac{d\log\frac{2dM}{\delta}}{n_{m,2}}}, \label{eq:lll-2-3}
	\end{align}  Furthermore, because $\mP_{\mB} \vtheta_m = \vtheta_m$, we have $\norm{\vtheta_m - \hatvb_{\tau_m}} = \norm{\mP_{\mB}(\vtheta_m - \hatvtheta_m)} \le \norm{\vtheta_m - \hatvtheta_m}$. Finally, we note that 
	\begin{align}
		\norm{\vd_{\tau_m}} &= \norm{\mP^\perp_{\mB}\hatvtheta_m} \notag \\
		&= \norm{\mP^\perp_{\mB}(\hatvtheta_m - \vtheta_{m})} \le \norm{\hatvtheta_m - \vtheta_{m}},
	\end{align}
	and we conclude by a union bound over $\ell \le M$. 
\end{proof}

\begin{lem} \label{lem:lll-con-3} Let $\kappa = d \log \frac{d}{\delta}$. Define the event $\gE_{3} = \{\mD_{\kappa}^\top \mD_\kappa \preccurlyeq \epsilon^2 \mI\}$. Then $\Pr[\gE_3] \ge 1 - \delta$.
\end{lem}

\begin{proof} Define the $\sigma$-field $\gF_\ell = \sigma(\tau_1, \cdots, \tau_\ell)$. Then $\gF = \{\gF_1 \subseteq \gF_2 \subseteq \cdots\}$ is a filtration. Note that $\vd_1, \vd_2, \cdots, \vd_\kappa$ is a stochastic process adapted to the filtration $\gF$, and that $\mD_\kappa^\top \mD_\kappa = \sum_{\ell = 1}^\kappa \vd_\ell \vd_\ell^\top$. Furthermore, we note that $
		\vd_\ell \mid \gF_\ell \sim \gN(0, \frac{1}{n_{m,2}}\mI)$, so by Lemma \ref{lem:matcon}, with probability $1-\delta$,
	\begin{align}
		\norm{\sum_{\ell=1}^\kappa \vd_\ell \vd_\ell^\top -  \frac{1}{n_{m,2}}\mI} \le O(\sqrt{\frac{\kappa}{n_{m,2}}\log \frac{d}{\delta}}).
	\end{align}
\end{proof}

\end{document}

%% file: math_commands.tex
\usepackage{amsmath,amsfonts,bm,physics}

\def\vbeta{{\bm{\beta}}}
\newcommand{\Unif}{{\mathrm{Unif}}}

\def\mTheta{{\bm{\Theta}}}

\def\hatB{{\widehat{\mB}}}
\def\hatW{{\widehat{\mW}}}
\def\hatvtheta{{\widehat{\vtheta}}}
\def\hatmTheta{{\widehat{\mTheta}}}
\def\hatvw{{\widehat{\vw}}}
\def\hatvb{{\widehat{\vb}}}
\def\hatmB{{\widehat{\mB}}}
\def\hatmW{{\widehat{\mW}}}
\def\hatmM{{\widehat{\mM}}}

\def\tildevtheta{{\widetilde{\bm{\theta}}}}
\def\tildevw{{\widetilde{\bm{w}}}}

\def\tildeO{{\widetilde{O}}}
\def\Reg{{\mathfrak R}}

\def\eqref#1{equation~\ref{#1}}

\def\1{\bm{1}}

\def\vtheta{{\bm{\theta}}}
\def\va{{\bm{a}}}
\def\vb{{\bm{b}}}

\def\vd{{\bm{d}}}
\def\ve{{\bm{e}}}

\def\vw{{\bm{w}}}
\def\vx{{\bm{x}}}

\def\vz{{\bm{z}}}

\def\mA{{\bm{A}}}
\def\mB{{\bm{B}}}

\def\mD{{\bm{D}}}

\def\mI{{\bm{I}}}

\def\mM{{\bm{M}}}

\def\mP{{\bm{P}}}

\def\mV{{\bm{V}}}
\def\mW{{\bm{W}}}

\def\mLambda{{\bm{\Lambda}}}

\DeclareMathAlphabet{\mathsfit}{\encodingdefault}{\sfdefault}{m}{sl}
\SetMathAlphabet{\mathsfit}{bold}{\encodingdefault}{\sfdefault}{bx}{n}

\def\gA{{\mathcal{A}}}

\def\gE{{\mathcal{E}}}
\def\gF{{\mathcal{F}}}

\def\gN{{\mathcal{N}}}

\def\sN{{\mathbb{N}}}

\def\sR{{\mathbb{R}}}
\def\sS{{\mathbb{S}}}

\newcommand{\E}{\mathbb{E}}

\newcommand{\R}{\mathbb{R}}

\DeclareMathOperator*{\argmax}{arg\,max}
\DeclareMathOperator*{\argmin}{arg\,min}